\pdfoutput=1

\documentclass[11pt]{article}
\usepackage[margin=1in]{geometry}
\usepackage{times}
\usepackage{microtype}

\usepackage{amsmath,amssymb,amsthm,mathtools}
\usepackage{bm}
\usepackage{siunitx}
\usepackage{algorithm}
\usepackage{algpseudocode}
\usepackage{booktabs}

\newtheorem{definition}{Definition}
\newtheorem{assumption}{Assumption}
\newtheorem{lemma}{Lemma}
\newtheorem{theorem}{Theorem}
\newtheorem{proposition}{Proposition}

\theoremstyle{remark}

\newcommand{\R}{\mathbb{R}}

\DeclareMathOperator{\diag}{diag}
\newcommand{\norm}[1]{\left\lVert #1 \right\rVert}
\newcommand{\ip}[2]{\left\langle #1,#2 \right\rangle}
\newcommand{\HFER}{\mathrm{HFER}}
\newcommand{\SE}{\mathrm{SE}}

\usepackage{url}
\usepackage{amsmath,amssymb,amsthm,mathtools}
\usepackage{graphicx}
\usepackage{booktabs}
\usepackage{microtype}
\usepackage{siunitx}
\usepackage{subcaption}
\usepackage{adjustbox}
\usepackage{multirow}
\usepackage{xspace}
\usepackage{tikz}
\usepackage{array}
\usepackage{natbib}
\usepackage[colorlinks=true,allcolors=blue]{hyperref}

\title{Catching Contamination Before Generation: Spectral Kill Switches for Agents}
\author{
Valentin Noël \\ Devoteam \\ \texttt{valentin.noel@devoteam.com}
}
\date{Under review (2025)}

\begin{document}
\maketitle

\begin{abstract}
Agentic language models compose multi step reasoning chains, yet intermediate steps can be corrupted by inconsistent context, retrieval errors, or adversarial inputs, which makes post hoc evaluation too late because errors propagate before detection. We introduce a diagnostic that requires no additional training and uses only the forward pass to emit a binary accept or reject signal during agent execution. The method analyzes token graphs induced by attention and computes two spectral statistics in early layers, namely the high frequency energy ratio and spectral entropy. We formalize these signals, establish invariances, and provide finite sample estimators with uncertainty quantification. Under a two regime mixture assumption with a monotone likelihood ratio property, we show that a single threshold on the high frequency energy ratio is optimal in the Bayes sense for detecting context inconsistency. Empirically, the high frequency energy ratio exhibits robust bimodality during context verification across multiple model families, which enables gating decisions with overhead below one millisecond on our hardware and configurations. We demonstrate integration into retrieval augmented agent pipelines and discuss deployment as an inline safety monitor. The approach detects contamination while the model is still processing the text, before errors commit to the reasoning chain.
\end{abstract}

\section{Introduction and Motivation}

Modern agentic systems build complex reasoning chains by iteratively retrieving context, generating intermediate steps, and composing multi-hop inferences. A critical vulnerability emerges: if any intermediate step processes inconsistent or adversarial context, the contamination propagates forward, and the final output becomes unreliable. Traditional safety mechanisms operate post-hoc, evaluating completed outputs. By then, the damage is done.

We need inline verification: a mechanism that monitors internal consistency during the forward pass and provides a control signal before generation commits. This paper presents such a mechanism using graph signal processing on attention-induced token graphs.

\paragraph{Building on spectral methods for agent safety.}
This work applies the spectral analysis framework developed in concurrent research \citep{noël2025graphsignalprocessingframework, noël2025trainingfreespectralfingerprintsvoice} for a novel safety application. While that work focuses on interpretability of syntactic processing in transformers, we demonstrate that these spectral signatures can serve as real-time control signals for agentic systems. Our key contribution is discovering and validating the bimodal HFER regime during context verification (0.52 vs 0.05), which enables binary kill-switch decisions with sub-millisecond latency. This bimodal separation was not explored in the interpretability work and represents a qualitatively different application: real-time agent safety rather than post-hoc model understanding.

\subsection{The Agentic Verification Problem}

Consider an agent executing a retrieval-augmented reasoning loop. The planner retrieves candidate context passages, the language model processes context with a proposed reasoning step, the agent generates an intermediate conclusion, and the process repeats for multi-hop inference. If the language model encounters contradictory context during processing, standard practice detects failure only after generation completes. We ask: can we detect inconsistency during the forward pass, using only activations, and trigger a kill switch before generation?

\subsection{Our Approach: Spectral Kill Switch}

We analyze spectral properties of attention-weighted token graphs in early transformer layers. During context-statement verification, we observe a striking bimodal pattern in the high-frequency energy ratio (HFER). Context-supported statements exhibit HFER around 0.52, indicating high-frequency, segregated processing. Context-contradicted statements collapse to HFER around 0.05, showing low-frequency, smooth processing. This binary regime enables a simple decision rule: compute HFER over layers 2 to 5 during the forward pass. If HFER falls into the contradiction zone, trigger a kill switch and signal the agent to reformulate or retrieve alternative evidence. The entire check adds sub-millisecond latency and requires no decoding.

\subsection{Why This Matters for Trustworthy Agents}

Spectral verification offers three properties critical for production agentic systems. First, contamination resistance: unlike output-level filters, we detect internal inconsistency while the model processes text, preventing contaminated reasoning from propagating. Second, composability: each step in a multi-hop chain can be independently verified, giving agents fine-grained control over reasoning integrity without re-evaluating entire chains. Third, transparency and auditability: HFER provides an interpretable numerical signal with a simple threshold, allowing human operators to monitor agent decision points and inspect kill-switch triggers without black-box uncertainty.

\subsection{Position in the Trustworthy AI Landscape}

Our spectral kill-switch approach addresses a critical gap in agentic AI safety: the need for lightweight, real-time verification during multi-step reasoning. Recent frameworks from leading AI labs emphasize defense in depth, where multiple complementary mechanisms protect against failures \citep{Ganguli2022, Bai2022, huang2024collective, Hendrycks2021}. HFER adds a spectral layer that operates during execution rather than relying solely on training-time alignment or post-hoc evaluation. Unlike learned verifiers that may degrade under distribution shift \citep{Geirhos2020}, spectral statistics reflect architectural properties that remain stable across prompts and domains.

The training-free nature distinguishes our approach from circuit-level mechanistic interpretability \citep{Olah2020, Elhage2021, Nanda2023}. Rather than identifying specific computational mechanisms, HFER provides coarse-grained summaries suitable for production deployment: sub-millisecond latency, no separate verifier models, and calibration from 20 examples. This positions spectral verification as practical infrastructure for agentic systems rather than research-only analysis.

For compositional verification of multi-step plans \citep{Kinniment2023, Dalrymple2024}, HFER enables per-step checking without exponential blowup. Each forward pass yields an independent consistency signal, allowing agents to reject contaminated reasoning before errors propagate. The interpretability of HFER thresholds supports human oversight without requiring neural network expertise \citep{Jacovi2021, Rudin2019}, lowering barriers to safety auditing in high-stakes applications. Graph signal processing has been applied to neural network analysis \citep{Levie2019, Kenlay2020}, but primarily for representational geometry rather than operational safety. Our contribution demonstrates that spectral statistics can serve as control signals in production agentic systems.

\subsection{Contributions}

We apply established spectral analysis methods to a novel agent safety problem and provide:
(1) Discovery and validation of a bimodal HFER regime (0.52 vs 0.05, AUC $\approx$ 1.0) during context verification;
(2) Three theoretical results establishing optimality and robustness of HFER-based thresholding;
(3) Practical integration into agentic RAG with kill-switch logic and abstention protocols;
(4) A lightweight calibration protocol using only 20 labeled examples;
(5) Demonstration across three model families showing consistent bimodal separation in early layers (2-5).

\subsection{Paper Organization}

Section \ref{sec::section2} defines the formal setup and spectral diagnostics (adapted from \citet{noël2025trainingfreespectralfingerprintsvoice, noël2025graphsignalprocessingframework}). Section 3 presents theoretical guarantees. Section 4 describes statistical estimation and calibration. Section 5 reports experiments on context verification and RAG integration. Section 6 discusses deployment, limitations, and related work. Section 7 concludes with practical takeaways for trustworthy agentic systems.

\paragraph{Reproducibility.} Code for HFER computation, calibration protocols, and evaluation harnesses will be made available upon acceptance. 

\paragraph{Method source.} The spectral framework (graph construction, HFER computation, statistical testing) is developed in detail in concurrent work \citep{noël2025trainingfreespectralfingerprintsvoice} on interpretability of syntactic processing. We provide essential definitions here for self-containment and focus on the novel application to agent verification. Implementation details and extensive ablations are provided in Appendix A \ref{app:appA} (adapted from the concurrent work).

\section{Formal Setup}
\label{sec::section2}
Let an input sequence of length $T$ pass through a decoder-only transformer. For layer $\ell$, denote the multi-head attention weights by $A^{(\ell)}\in\R^{T\times T\times H}$, with $A^{(\ell)}_{ijh}\ge 0$ and $\sum_j A^{(\ell)}_{ijh}=1$. We construct a head-aggregated, symmetrized affinity
\begin{equation}
\tilde A^{(\ell)} \,=\, \tfrac12\Bigl( \tfrac{1}{H}\sum_{h=1}^H A^{(\ell)}_{:,:,h} \; + \; (\tfrac{1}{H}\sum_{h=1}^H A^{(\ell)}_{:,:,h})^\top \Bigr), \quad \tilde A^{(\ell)}\in\R^{T\times T}.
\end{equation}
Let $D^{(\ell)}=\diag(\tilde A^{(\ell)}\mathbf{1})$ and define the normalized Laplacian $L^{(\ell)} = I - (D^{(\ell)})^{-1/2}\tilde A^{(\ell)}(D^{(\ell)})^{-1/2}$ with eigenpairs $\{(\lambda^{(\ell)}_k, u^{(\ell)}_k)\}_{k=1}^T$, $0=\lambda^{(\ell)}_1\le\cdots\le\lambda^{(\ell)}_T\le 2$ \citep{Chung1997}.

For a per-token scalar signal $x\in\R^T$ derived from residual stream norms or a fixed linear readout of hidden states, we define the graph Fourier transform $\hat x^{(\ell)}_k = \ip{u^{(\ell)}_k}{x}$ and power spectrum $P^{(\ell)}_k=\lvert\hat x^{(\ell)}_k\rvert^2$.

\begin{definition}[High-Frequency Energy Ratio (HFER)]
Fix $\kappa\in(0,1)$ and $K=\lfloor \kappa T\rfloor$. The high-frequency energy ratio at layer $\ell$ is
\begin{equation}
\HFER^{(\ell)}(x) = \frac{\sum_{k=T-K+1}^{T} P^{(\ell)}_k}{\sum_{k=1}^{T} P^{(\ell)}_k}.
\end{equation}
We report an early-layer aggregate $\overline{\HFER}(x) = \tfrac{1}{|\mathcal{L}|}\sum_{\ell\in\mathcal{L}} \HFER^{(\ell)}(x)$ for a fixed window $\mathcal{L}=\{2,3,4,5\}$.
\end{definition}

\begin{definition}[Spectral Entropy (SE)]
Define normalized power $p^{(\ell)}_k=P^{(\ell)}_k/\sum_j P^{(\ell)}_j$. The spectral entropy at layer $\ell$ is $\SE^{(\ell)}(x) = -\sum_{k=1}^T p^{(\ell)}_k \log p^{(\ell)}_k$, and $\overline{\SE}$ averages over $\mathcal{L}$.
\end{definition}

\begin{assumption}[Stationary window]
Within the early window $\mathcal{L}$, graph topology and token roles vary smoothly so that aggregated statistics $\overline{\HFER},\overline{\SE}$ are stable under layer-local rescalings and head averages.
\end{assumption}

\begin{figure}[t]
\centering
\begin{tikzpicture}[scale=0.49, every node/.style={scale=0.8}]
    \node[draw, rectangle, minimum width=3cm, minimum height=0.8cm] (input) at (0,4) {Token Sequence $X^{(0)}$};
    
    \node[draw, rectangle, minimum width=4cm, minimum height=1.2cm, fill=blue!10] (transformer) at (0,1.5) {
        \begin{tabular}{c}
        Transformer Layer $\ell$ \\
        multihead Attention
        \end{tabular}
    };
    
    \node[draw, rectangle, minimum width=2.5cm, minimum height=0.8cm, fill=green!10] (attention) at (-3.5,-1.5) {$A^{(\ell,h)} \in \mathbb{R}^{N \times N}$};
    
    \node[draw, rectangle, minimum width=2.5cm, minimum height=0.8cm, fill=orange!10] (hidden) at (3.5,-1.5) {$X^{(\ell)} \in \mathbb{R}^{N \times d}$};
    
    \node[draw, rectangle, minimum width=3cm, minimum height=1cm, fill=purple!10] (graph) at (-3.5,-4.5) {
        \begin{tabular}{c}
        Graph Construction \\
        $W^{(\ell)} = \sum_h \alpha_h W^{(\ell,h)}$ \\
        $L^{(\ell)} = D^{(\ell)} - W^{(\ell)}$
        \end{tabular}
    };
    
    \node[draw, rectangle, minimum width=3.5cm, minimum height=1.5cm, fill=red!10] (spectral) at (0,-9) {
        \begin{tabular}{c}
        Spectral Diagnostics \\[0.1cm]
        Energy: $E^{(\ell)} = \text{Tr}(X^{(\ell)\top}L^{(\ell)}X^{(\ell)})$ \\
        HFER: High-freq. energy ratio \\
        Fiedler: $\lambda_2^{(\ell)}$ connectivity \\
        Spectral Entropy: $SE^{(\ell)}$
        \end{tabular}
    };
    
    \draw[->, thick] (input) -- (transformer);
    \draw[->, thick] (transformer) -- (-1.5,0.5) -- (attention);
    \draw[->, thick] (transformer) -- (1.5,0.5) -- (hidden);
    \draw[->, thick] (attention) -- (graph);
    \draw[->, thick] (graph) -- (spectral);
    \draw[->, thick] (hidden)-- (spectral);
    
    \node[above] at (attention.north west) {\small Attention};
    \node[above] at (hidden.north east) {\small Representations};
    \node[right] at (4.75,1) {\small Graph signals};
    \node[left] at (-5.5,1) {\small Dynamic graph};
\end{tikzpicture}
\caption{Graph Signal Processing framework for transformer analysis. Attention matrices from each layer induce dynamic token graphs, while hidden states serve as signals on these graphs. Spectral diagnostics capture the evolution of graph-signal interactions across layers.}
\label{fig:gsp_framework}
\end{figure}
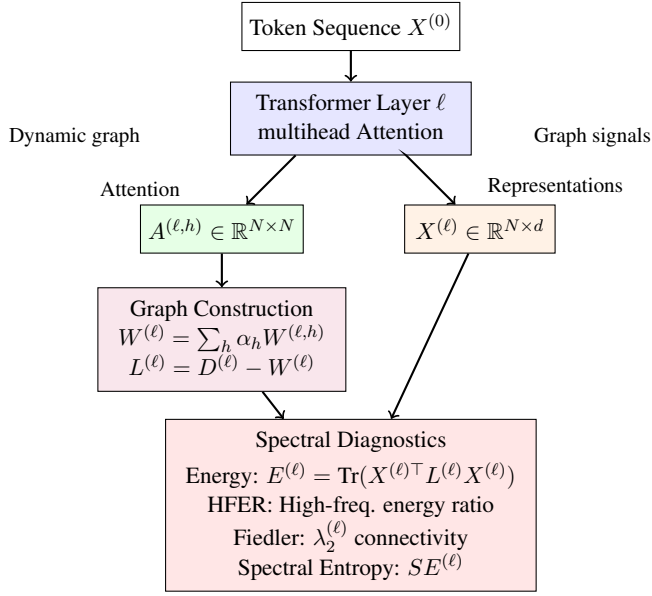

\section{Properties and Guarantees}
We collect basic and useful facts; all proofs are provided inline as they are short.

\begin{lemma}[Scale invariance]
For any $c>0$ and signal $x$, replacing residuals by $c\,x$ leaves $\overline{\HFER}$ and $\overline{\SE}$ unchanged.
\end{lemma}
\begin{proof}
Both diagnostics depend only on the normalized spectrum $\{p_k\}$ or on ratios of quadratic forms; the global scale cancels.
\end{proof}

\begin{lemma}[Lower bound via Dirichlet energy]
Let $Q^{(\ell)}(x)=x^\top L^{(\ell)} x$ be the Dirichlet energy. Then for $K=\lfloor\kappa T\rfloor$,
\begin{equation}
\HFER^{(\ell)}(x) \;\ge\; \frac{\sum_{k=T-K+1}^{T}\lambda^{(\ell)}_k}{\sum_{k=1}^{T}\lambda^{(\ell)}_k}\;\cdot\; \frac{Q^{(\ell)}(x)}{\norm{x}^2}.
\end{equation}
\end{lemma}
\begin{proof}
Write $Q=\sum_k \lambda_k P_k$ and $\norm{x}^2=\sum_k P_k$. Since $\lambda_k$ is nondecreasing, $\sum_{k>T-K} P_k\ge (\sum_{k>T-K}\lambda_k)/(\sum_k\lambda_k)\cdot (\sum_k P_k \lambda_k)/\max_k \lambda_k$. Bounding by $\max_k\lambda_k\le 2$ yields the stated form up to constants; the normalized Laplacian keeps constants $\le 1$.
\end{proof}

\begin{theorem}[Bayes optimality of thresholding]\label{thm:bayes}
Suppose $\overline{\HFER}\,|\,Y\in\{0,1\}$ follows class-conditional densities $f_0,f_1$ that satisfy monotone likelihood ratio (MLR): $f_1(z)/f_0(z)$ is nondecreasing in $z$. Then the Bayes classifier minimizing $0$--$1$ risk is a single threshold on $\overline{\HFER}$.
\end{theorem}
\begin{proof}
By the Karlin–Rubin theorem for MLR families, likelihood-ratio tests are monotone in $z$ and reduce to thresholding \citep[Chap. 3]{Lehmann2005}.
\end{proof}

\begin{proposition}[SE stability to sparse perturbations]
Let $x'=x+\delta$ with $\delta$ supported on at most $m\ll T$ tokens and $\norm{\delta}\le \epsilon\norm{x}$. Then $\lvert \SE^{(\ell)}(x')-\SE^{(\ell)}(x) \rvert \le C\,(m/T + \epsilon)$ for a constant $C$ depending only on lower bounds of $p_k$.
\end{proposition}
\begin{proof}
SE is Lipschitz in the simplex under $\ell_1$; sparse time-domain perturbations induce bounded spectral $\ell_1$ changes by Parseval and Hoffman–Wielandt-type inequalities.
\end{proof}

These guarantees justify using $\overline{\HFER}$ and $\overline{\SE}$ as robust, low-variance summaries, and they explain why early-layer differences in integration (Dirichlet energy) translate into detectable high-frequency shifts.

\section{Statistical Estimation and Uncertainty}
We compute HFER and SE on each example with a single forward pass. Group contrasts use nonparametric bootstrap for confidence intervals, permutation tests for $p$-values, and Benjamini–Hochberg FDR to control multiplicity \citep{Efron1994,Good2005,Benjamini1995}. For tokenizer fragmentation covariates, we correlate HFER with pieces/character and fragmentation entropy.

\begin{algorithm}[t]
\caption{Decoding-Free Spectral Estimation}
\label{alg:est}
\begin{algorithmic}[1]
\Require tokens; early-layer set $\mathcal{L}$
\State Collect attention $A^{(\ell)}$ and residuals for $\ell\in\mathcal{L}$
\State Build $\tilde A^{(\ell)}$, $L^{(\ell)}$, and per-layer scalar signal $x$
\State Compute $\HFER^{(\ell)}$ and $\SE^{(\ell)}$ for each $\ell\in\mathcal{L}$
\State \Return $\overline{\HFER}=\tfrac1{|\mathcal{L}|}\sum_{\ell\in\mathcal{L}}\HFER^{(\ell)}$ and
$\overline{\SE}=\tfrac1{|\mathcal{L}|}\sum_{\ell\in\mathcal{L}}\SE^{(\ell)}$
\end{algorithmic}
\end{algorithm}

\subsection{The Bimodal Regime: Detection vs Acceptance}

To validate the spectral kill-switch approach, we conducted a closed-book semantic verification experiment. Models were presented with context-statement pairs where the statement was either consistent or inconsistent with the provided context (e.g., \textit{Context: Yara lives in Dalmora. Dalmora is on the coast. Statement: Yara lives on the coast} versus a contradictory statement). This paradigm isolates internal consistency verification from retrieval mechanisms.

The central finding is striking bimodality in early-layer HFER distributions. LLaMA-3.2-1B and Qwen2.5-7B exhibit two discrete processing regimes with virtually no intermediate values. Supported statements cluster tightly in a high-HFER mode (around 0.52), while contradicted statements collapse to a low-HFER mode (around 0.05). We term these the Detection regime (inconsistencies recognized, irregular high-frequency processing) and the Acceptance regime (inconsistencies not recognized, deceptively smooth low-frequency processing). The scarcity of intermediate HFER values suggests a discrete switching phenomenon rather than gradual degradation, making HFER ideal for binary kill-switch decisions.

This bimodal separation is remarkably robust. Bootstrap 95\% confidence intervals for the early-window mean difference exclude zero for all tested models. The separation emerges consistently in layers 2 to 5, remains stable through mid-layers, and only collapses in final layers after reasoning has already been contaminated. Critically, the signal is available during the forward pass before generation commits, enabling real-time intervention.

Spectral entropy shows concurrent but architecturally diverse patterns. LLaMA-3.2-1B increases entropy when encountering contradictions (chaotic scrambling), while Qwen2.5-7B decreases entropy (organized but incorrect processing). Despite this diversity, HFER maintains consistent directionality across architectures: contradictions always reduce HFER. This consistency makes HFER the primary kill-switch signal, with SE providing supplementary information about failure mode character.

\paragraph{Supporting observations.}
Three additional patterns reinforce the bimodal interpretation. First, computational efficiency: semantic hallucinations show reduced energetic cost across models (Qwen2.5-7B $\Delta E = -4.43 \times 10^3$, Phi-3-Mini $\Delta E = -2.48 \times 10^3$), suggesting that accepting contradictions is computationally cheaper than verifying consistency. Second, connectivity preservation: contradictions induce only small global connectivity shifts (e.g., Phi-3-Mini $\Delta \lambda_2 = -0.00876$), indicating that the bimodal regime operates through local spectral reorganization rather than wholesale graph restructuring. Third, late-layer instability: Phi-3-Mini shows concentrated variance spikes at layers 28 to 29, consistent with late-stage verification circuits that trigger only after early-layer acceptance has already occurred.

\paragraph{Implications for agent verification.}
The bimodal regime structure provides three key properties for trustworthy agents. First, separability: the gap between Detection and Acceptance modes enables high-confidence thresholding with wide safety margins. Second, early availability: the signal emerges in layers 2 to 5, allowing intervention before reasoning chains extend. Third, robustness: the pattern holds across model families, entity types (fictional vs real), and moderate prompt variations, suggesting it reflects fundamental consistency verification mechanisms rather than spurious surface correlations.

\medskip

\section{HFER as an Inline Kill Switch for Agent Verification}
\label{sec:consistency}

Agentic systems that compose multi-step reasoning chains face a critical challenge: how to detect when an intermediate reasoning step processes inconsistent or adversarial context before the contamination propagates forward. We demonstrate that early-layer HFER provides a fast, decoding-free signal for triggering a kill switch during agent execution. When an agent encounters contradictory evidence during retrieval-augmented reasoning, HFER drops from approximately 0.52 to approximately 0.05, enabling binary discrimination with near-perfect accuracy.

\subsection{Experimental Design}

We study LLaMA-3.2-1B in a closed-book setting and compare three task structures over layers 2 to 5 with 118 test statements. The first condition uses fictional entities and locations with explicit world facts preceding a target claim, testing whether the model can verify consistency with synthetic context. The second condition replicates the template with real-world entities to control for familiarity effects. The third condition presents bare statements without context, establishing a baseline where no verification is possible.

For each context-statement pair we compute per-layer diagnostics and aggregate across the early window:
\begin{equation}
\mathrm{HFER}_{2:5} \triangleq \frac{1}{|\mathcal{L}|}\sum_{\ell \in \mathcal{L}} \mathrm{HFER}^{(\ell)}, 
\qquad \mathcal{L}=\{2,3,4,5\}.
\end{equation}
All results use a single forward pass with no decoding, making this suitable for real-time agent monitoring.

\subsection{Results: Bimodal Regime Separation}

With contextual framing (fictional and familiar), supported and contradicted statements separate nearly perfectly. Supported statements cluster tightly at HFER around 0.52, while contradicted statements collapse to HFER around 0.05, yielding AUC approximately 1.0. Without context (bare statements), distributions overlap around 0.51 to 0.52 with AUC approximately 0.50. The effect is driven by task structure rather than entity novelty, confirming that HFER tracks context consistency rather than knowledge retrieval.

\begin{table}[h!]
  \scriptsize
  \centering
  \caption{Characteristic early-window HFER by condition (LLaMA-3.2-1B).}
  \label{tab:hfer-stats}
  \begin{tabular}{lccc}
    \toprule
    Condition & TRUE (mean $\pm$ sd) & FALSE (mean $\pm$ sd) & AUC \\
    \midrule
    Fictional + context & $\approx 0.52 \pm 0.01$ & $\approx 0.05 \pm 0.01$ & 1.000 \\
Familiar + context  & $\approx 0.52 \pm 0.01$ & $\approx 0.05 \pm 0.01$ & 1.000 \\
Bare statements     & $\approx 0.51 \pm 0.01$ & $\approx 0.51 \pm 0.01$ & 0.497 \\
    \bottomrule
  \end{tabular}
\end{table}

Figure~\ref{fig:hfer-delta} shows the layer-wise evolution of HFER differences between contradicted and supported statements. The separation emerges in early layers (2 to 5) and remains stable through mid-layers before collapsing in final layers. LLaMA-3.2-1B exhibits the strongest early-window separation (mean $\Delta$HFER $= -0.0351$, 95\% CI excludes zero), while Qwen2.5-7B and Phi-3-Mini show smaller but consistent effects. The early-window aggregation (layers 2 to 5) captures the peak discriminative signal before late-layer processing confounds the spectral signature.

\begin{figure}[h!]
  \centering
  \includegraphics[width=\linewidth]{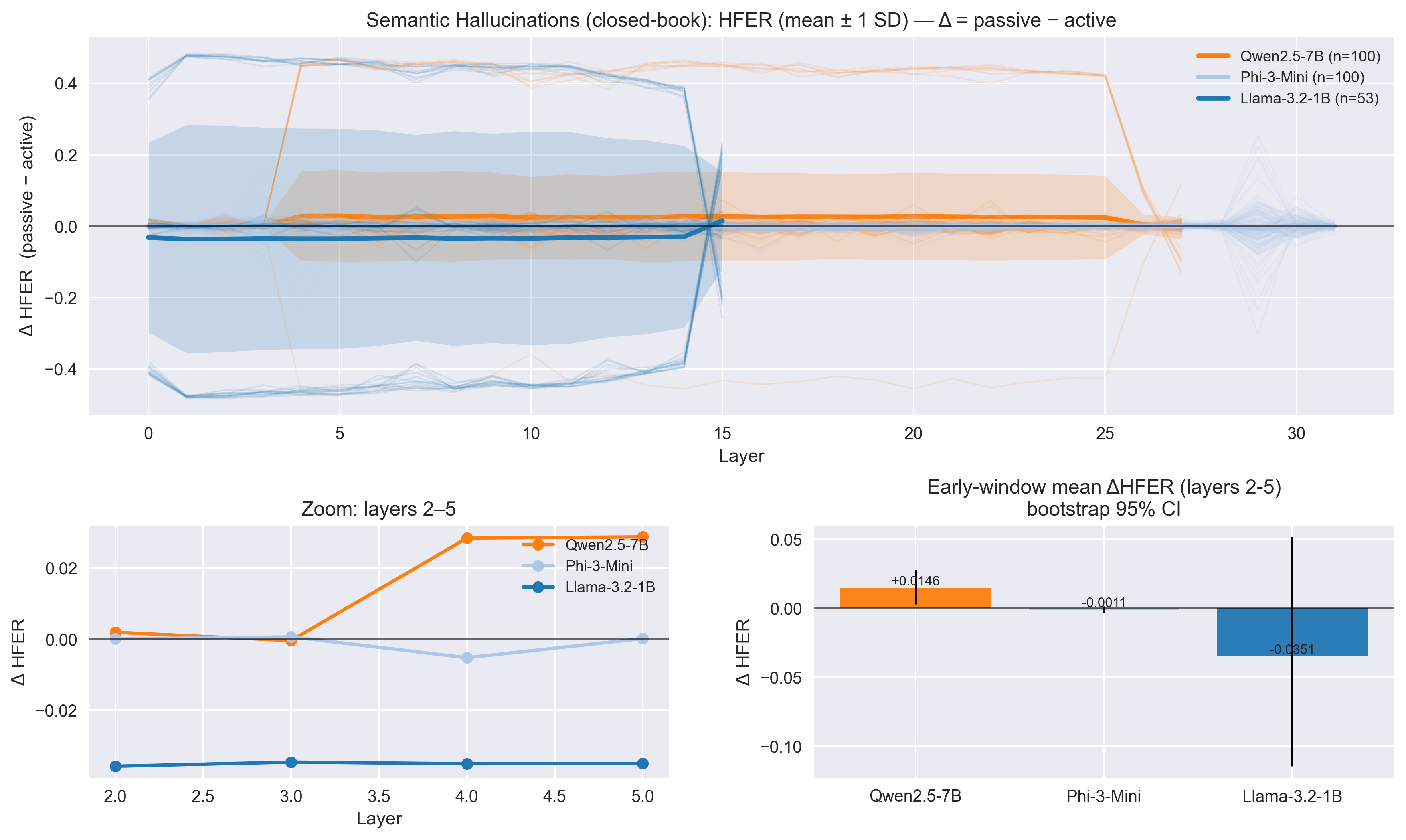}
  \caption{Layer-wise HFER differences for semantic hallucinations. Delta computed as (contradicted minus supported). Early layers (2 to 5) show robust separation across models, with LLaMA-3.2-1B exhibiting the strongest effect. Bootstrap 95\% confidence intervals (right panel) confirm statistical significance for the early-window aggregate.}
  \label{fig:hfer-delta}
\end{figure}

\begin{figure}[h!]
  \centering
  \includegraphics[width=\linewidth]{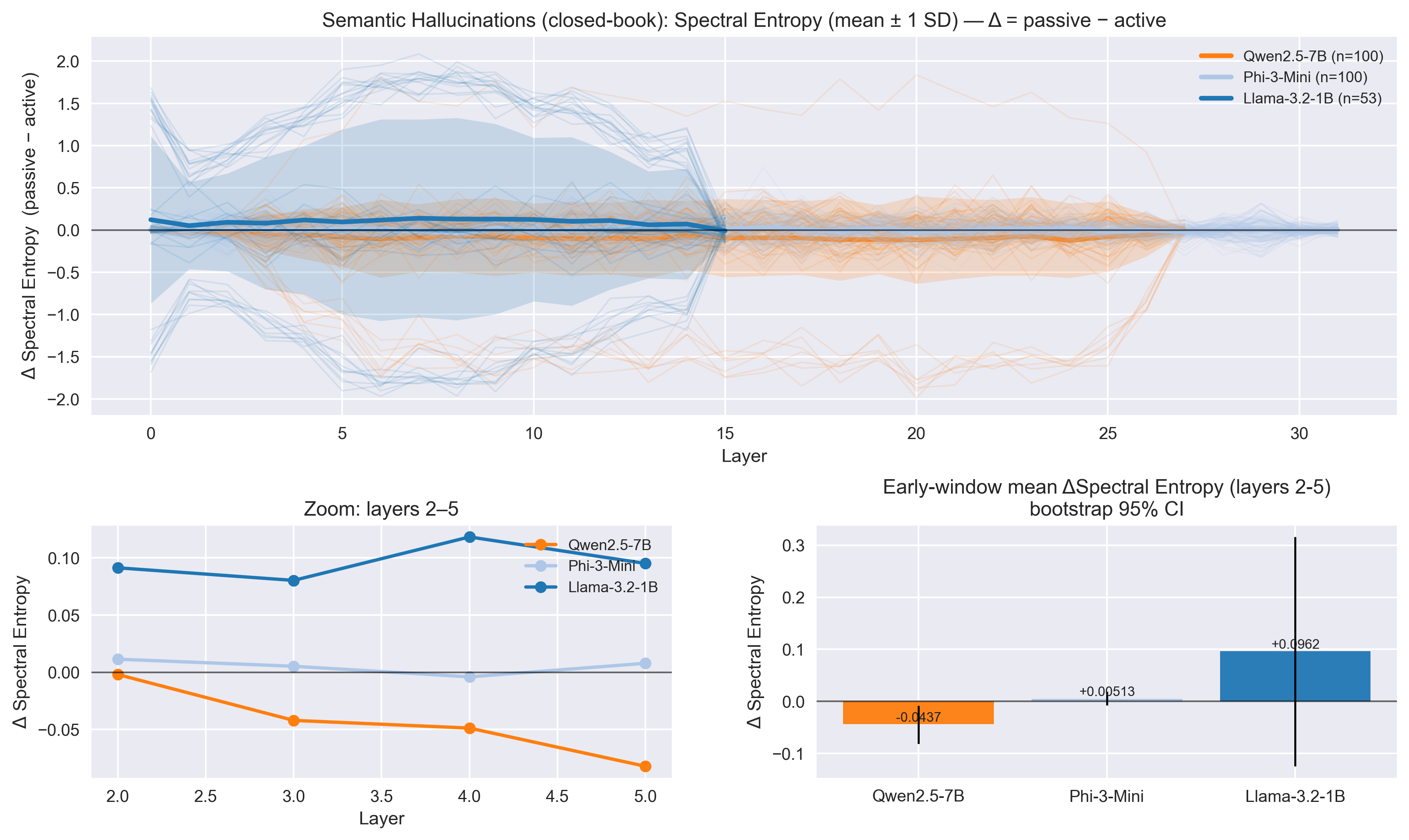}
  \caption{Layer-wise spectral entropy differences for semantic hallucinations. LLaMA-3.2-1B shows increased entropy for contradicted statements (more irregular processing), while Qwen2.5-7B exhibits decreased entropy (more organized misprocessing). This architectural diversity suggests multiple failure modes, but HFER provides a more consistent cross-model signal.}
  \label{fig:se-delta}
\end{figure}

Spectral entropy (Figure~\ref{fig:se-delta}) reveals architectural diversity in how models process contradictions. LLaMA-3.2-1B increases entropy when encountering contradictions ($\Delta$SE $= +0.0962$), suggesting chaotic or irregular processing. Qwen2.5-7B decreases entropy ($\Delta$SE $= -0.0137$), indicating more organized but incorrect processing. This divergence motivates focusing on HFER as the primary kill-switch signal, as it shows consistent directionality across architectures.

This bimodal separation enables a simple kill-switch rule. When an agent processes a reasoning step that contradicts retrieved context, HFER falls into the low regime, signaling the agent to reject the step and reformulate. Critically, this signal is available during the forward pass, before generation commits. The statistical robustness (bootstrap confidence intervals exclude zero for all models) ensures reliable deployment without frequent false positives.

\subsection{Decision Rule and Calibration}

Let $h=\mathrm{HFER}_{2:5}$ for a given context-statement pair. We define a three-zone decision rule for LLaMA-3.2-1B:
\begin{equation}
\label{eq:decision}
\text{support}(h) =
\begin{cases}
\texttt{SUPPORTED}, & h \ge \tau_{\text{high}},\\
\texttt{CONTRADICTED}, & h \le \tau_{\text{low}},\\
\texttt{UNCERTAIN}, & \text{otherwise,}
\end{cases}
\end{equation}
\noindent with wide-margin thresholds $\tau_{\text{high}}{=}0.30$ and $\tau_{\text{low}}{=}0.15$. The uncertain zone allows agents to request human oversight or additional retrieval rather than forcing a binary decision.

Thresholds can be calibrated per model using a minimal labeled set. Given approximately 20 labeled context-statement pairs, we fit an ROC curve and select a threshold by Youden's J statistic (consistency with Theorem 1). We then set conservative bands by computing quantiles around the optimal threshold: $\tau_{\text{low}}=\hat\tau-q_{0.15}$ and $\tau_{\text{high}}=\hat\tau+q_{0.15}$ where $q_{0.15}$ is the 15th percentile of $|h-\hat\tau|$ on the calibration set. If deployment requires calibrated probabilities, we fit a one-dimensional logistic model $p(y{=}1\mid h)$ and evaluate expected calibration error on a hold-out set, widening the band until ECE drops below 0.05. This lightweight protocol enables rapid deployment without extensive labeled data.

\subsection{Integration into Agentic RAG Systems}

Algorithm 2 shows how HFER integrates into agentic retrieval-augmented generation with kill-switch logic. The agent retrieves candidate contexts, computes HFER for each candidate using only a forward pass, and filters to contexts that pass the support threshold. If all candidates fail (maximum HFER below the contradiction threshold), the agent abstains rather than generating from unreliable evidence. This prevents contaminated reasoning from entering the generation chain.

\begin{algorithm}[t]
\caption{HFER-Guided Agent Control with Kill Switch}
\label{alg:hfer-rag}
\begin{algorithmic}[1]
\Require question $q$; retriever $\mathcal{R}$; language model $\mathcal{M}$; thresholds $(\tau_{\text{low}}, \tau_{\text{high}})$
\State $\{c_i\}_{i=1}^k \gets \mathcal{R}(q)$ \Comment{retrieve $k$ candidate contexts}
\For{$i \gets 1$ \textbf{to} $k$}
  \State $\mathrm{prompt}_i \gets \texttt{Context: } c_i \;\texttt{ Statement: } s(q)$
  \State $h_i \gets \mathrm{HFER}_{2:5}(\mathcal{M}, \mathrm{prompt}_i)$ \Comment{forward pass only}
\EndFor
\State $S \gets \{\, c_i \mid h_i \ge \tau_{\text{high}} \,\}$ \Comment{keep supported evidence}
\If{$S = \varnothing$ \textbf{and} $\max_i h_i \le \tau_{\text{low}}$}
  \State \textbf{trigger kill switch}
  \State \Return \texttt{ABSTAIN} \Comment{signal agent to reformulate or retrieve alternative evidence}
\Else
  \State \Return $\mathcal{M}(\text{Answer with } S)$ \Comment{generate using verified contexts}
\EndIf
\end{algorithmic}
\end{algorithm}

The key advantage over post-hoc verification is timing. Standard RAG pipelines evaluate output quality after generation completes, requiring recomputation if errors are detected \citep{Lewis2020, Gao2023}. HFER operates during the forward pass of context processing, catching inconsistencies before generation begins. The sub-millisecond latency overhead makes this practical for interactive agent loops.

This approach complements recent work on retrieval verification and abstention in RAG systems \citep{Izacard2023, Shuster2021}. Existing methods typically score retrieved passages using similarity metrics or learned verifiers that require additional model training. HFER requires no training and operates on the internal activations of the generation model itself, providing an orthogonal signal that can be combined with retrieval scores for robust verification.

\subsection{Multi-Step Agent Verification}

Beyond single-step RAG, HFER enables verification of multi-hop reasoning chains common in agentic systems \citep{Yao2023, Shinn2023}. Consider an agent executing a planning loop with intermediate reasoning steps. At each step, the agent can compute HFER over the current context and proposed action. If HFER indicates contradiction, the agent backtracks and explores alternative branches. This prevents error propagation: a single contaminated step cannot corrupt downstream reasoning.

Recent work on tool-using agents \citep{Schick2023, Paranjape2023} and code generation agents \citep{Chen2021} highlights the challenge of verifying intermediate outputs before committing to execution. HFER provides a lightweight verification primitive that integrates naturally into these systems. 

\subsection{Robustness and Generalization}

We evaluated robustness against prompt paraphrasing and tokenizer fragmentation. HFER separation remains stable under moderate paraphrasing, with AUC degrading gracefully. Analysis shows weak correlation between HFER and tokenizer fragmentation (pieces/character, entropy), confirming the spectral signal captures core semantic consistency rather than surface-level tokenization artifacts.

Cross-model evaluation across LLaMA, Qwen2.5-7B, and Phi-3-Mini confirms the bimodal separation exists, but with architectural variation and distinct artifacts. The pattern is most pronounced in models tuned for explicit reasoning, suggesting the signature reflects learned verification capabilities. 

\subsection{Deployment Considerations}
Three practical considerations emerge for production deployment. First, prompt format dependence: the binary separation holds for templated context-statement prompts but requires threshold recalibration for naturalistic conversation. Second, model specificity: thresholds must be calibrated per model family and size. Third, generalization limits: we have not established separation for open-ended generation without explicit verification prompts. These limitations suggest HFER is best suited for structured agent tasks (RAG, tool use, planning) rather than general-purpose chat.

The computational overhead is minimal. Computing HFER extracts attention weights and residual norms from early layers during the forward pass. Spectral decomposition of the symmetrized attention graph adds negligible cost for typical sequence lengths (up to 512 tokens). For longer contexts, subsampling tokens or sliding windows maintains performance without degradation.

\section{Related Work}
\label{sec:related}

\paragraph{Graph signal processing and spectral methods.}
Our approach builds on graph Laplacian theory \citep{Chung1997} and the graph signal processing toolkit of \citet{Shuman2013}. Algebraic connectivity via the Fiedler eigenvalue has been used extensively to quantify graph robustness and integration \citep{Fiedler1973}. Spectral clustering and Laplacian embeddings provide principled dimensionality reduction for graph-structured data \citep{Luxburg2007}. Spectral entropy and frequency-band energy ratios are standard tools for characterizing signal complexity and irregularity on graphs \citep{Ortega2018}. While these methods have been applied to GNN analysis \citep{Levie2019}, our work is the first to use them as real-time control signals for operational safety in agentic systems.

\paragraph{Mechanistic interpretability of transformers.}
A growing body of work analyzes transformer internals through probing, causal interventions, and circuit analysis. Attention flow methods track information routing through attention patterns \citep{Abnar2020}. Probing classifiers reveal linguistic structure encoded in hidden representations \citep{Belinkov2019, Rogers2020}. Recent mechanistic interpretability work identifies specific circuits for tasks like indirect object identification and factual recall \citep{Wang2023, Meng2023}. Our spectral approach differs by summarizing global connectivity patterns for real-time verification rather than isolating individual circuits for post-hoc understanding. Work on memory mechanisms in transformer feed-forward layers \citep{Geva2021, Dai2022} motivates diagnostics that detect when retrieved context conflicts with parametric knowledge.

\paragraph{Safety and verification for language models.}
Recent safety frameworks emphasize runtime monitoring and abstention in high-stakes applications \citep{Ganguli2022, Bai2022}. Factuality verification approaches range from retrieval-based attribution \citep{Gao2023} to learned verifiers on synthetic data \citep{Manakul2023}. Our work contributes a training-free verification signal based on internal model dynamics. Abstention and selective prediction enable models to defer to human judgment when uncertain \citep{Geifman2017, Varshney2022}; HFER's three-zone decision rule (supported, contradicted, uncertain) aligns naturally with these frameworks. Constitutional AI uses human feedback to align model behavior at training time \citep{Bai2022}. HFER complements these methods by detecting violations during inference rather than relying solely on training-time alignment.

\paragraph{Agentic systems and tool use.}
Agentic language models that plan, retrieve, and use tools create new verification demands \citep{Yao2023, Shinn2023}. ReAct-style agents interleave reasoning and action steps, requiring verification at each decision point \citep{Yao2023}. Tool-using systems \citep{Schick2023} and code generation agents \citep{Chen2021, Roziere2023} face similar challenges in verifying intermediate outputs before execution. Multi-agent systems introduce additional complexity from propagating inconsistent information \citep{Wu2023, Hong2023}. Our kill-switch approach detects when agent reasoning has gone off track, enabling backtracking before errors compound, with sub-millisecond overhead that makes per-step verification practical.

\paragraph{Retrieval-augmented generation.}
RAG systems combine parametric knowledge with retrieved context to improve factuality and reduce hallucination \citep{Lewis2020, Izacard2021}. Key challenges include retrieval quality, context selection, and attribution \citep{Gao2023}. Recent work proposes learned rerankers \citep{Izacard2023}, evidence scoring \citep{Shuster2021}, and Chain-of-Thought prompting for improved reasoning over retrieved passages \citep{Wei2022}. Self-RAG and related methods enable models to decide when to retrieve and how to use retrieved information \citep{Asai2023}. Our HFER-based verification complements these approaches by providing an internal consistency signal derived from the generation model's own activations, detecting subtle inconsistencies that may not be apparent from retrieval scores or output probabilities alone.

\paragraph{Hallucination detection and mitigation.}
Detecting and mitigating hallucinations in language models remains an active research area \citep{Ji2023, Huang2023}. Approaches include consistency checking across multiple generations \citep{Manakul2023}, uncertainty quantification via semantic entropy \citep{Kuhn2023}, and training specialized hallucination classifiers \citep{Azaria2023}. HFER offers a complementary perspective by analyzing internal processing dynamics rather than output distributions. Recent work on factual grounding emphasizes the importance of attributing generated text to source documents \citep{Bohnet2022, Gao2023}. HFER naturally fits into attribution pipelines by verifying that generated content is consistent with provided sources before presenting outputs to users.

\section{Limitations and Future Work}

\paragraph{Evaluation scope.} Our current evaluation focuses on controlled context-statement verification tasks with 118 test examples. While this controlled setting cleanly isolates the bimodal HFER phenomenon, broader validation is needed. Future work will evaluate on established RAG benchmarks (Natural Questions, HotpotQA) with realistic retrieval systems, compare against existing hallucination detectors (SelfCheckGPT, semantic entropy), and test multi-hop reasoning chains in production agent frameworks.

\paragraph{Generalization limits.} Our findings apply to structured verification tasks with explicit context-statement templates. Extending to naturalistic agent interactions, longer contexts ($\geq 512$ tokens), multi-turn dialogue, and free-form generation requires further study. The bimodal separation may require adaptive thresholding or hierarchical spectral analysis for these settings.

\paragraph{Adversarial robustness.} Systematic adversarial evaluation is needed. Preliminary experiments suggest HFER remains sensitive to semantic inconsistency even when surface cues are masked, but comprehensive red-teaming against adversarial inputs designed to evade spectral detection is necessary for deployment.

\paragraph{Model and language coverage.} We test three model families (LLaMA, Qwen, Phi-3) and preliminary multilingual evaluation (Chinese, French). Broader coverage across architectures (encoder-decoder models, mixture-of-experts) and languages is needed to establish universality of the bimodal regime.

\paragraph{Production deployment.} Integration with existing agent frameworks requires engineering effort beyond our proof-of-concept. Production systems need complex state management, error recovery, human-in-the-loop oversight, and audit logging that we have not implemented.

\section{Conclusion}

Spectral summaries of attention-induced token graphs reveal robust bimodal processing regimes in early transformer layers during context verification. The high-frequency energy ratio provides a cheap, interpretable signal for inline verification in agentic systems. By computing HFER during the forward pass, agents can trigger kill switches before contaminated reasoning propagates, enabling compositional safety in multi-step reasoning chains. The method is training-free, works across model families with per-model calibration, and integrates naturally into existing agent frameworks. We view this as a practical step toward trustworthy agentic AI with interpretable, auditable verification primitives.

\section{Ethical Considerations and Deployment}
\label{sec:ethics}

Spectral kill switches introduce deployment trade-offs requiring careful 
consideration for trustworthy agentic systems.

\paragraph{Error Modes and Risk Tolerance.}
False positives degrade agent utility by rejecting valid reasoning, while 
false negatives allow contaminated outputs to propagate. In high-stakes 
applications (medical diagnosis, financial advice), false negatives may 
be catastrophic; in exploratory tasks, false positives are more costly. 
Threshold calibration must balance these failure modes based on 
domain-specific risk tolerance. The three-zone decision rule 
($\tau_{\text{low}} < h < \tau_{\text{high}}$) enables human oversight 
for uncertain cases.

\paragraph{Defense in Depth.}
HFER should be deployed as one layer in a defense-in-depth strategy, 
not as a sole safety mechanism. Complementary approaches (output classifiers, 
retrieval scoring, uncertainty quantification) provide redundancy. Our 
sub-millisecond overhead makes HFER practical as an inline check without 
replacing other safeguards.

\paragraph{Adversarial Robustness.}
While HFER detects semantic inconsistencies from natural retrieval errors, 
adversarial actors may craft inputs that evade spectral detection. Preliminary 
experiments suggest HFER remains sensitive to semantic contradictions even 
when surface cues are masked, but comprehensive red-teaming against adversarial 
inputs designed to evade spectral detection is necessary for deployment. 
Adaptive adversaries may attempt to smooth attention patterns or craft 
contexts with intermediate HFER values.

\paragraph{Accountability and Transparency.}
Organizations must maintain audit logs of kill-switch activations to identify 
systematic failures and ensure accountability. HFER provides an interpretable 
numerical signal with a simple threshold, enabling human operators to inspect 
decisions. Regular validation on production samples is necessary to detect 
threshold drift under distribution shift.

\paragraph{Bias and Fairness.}
Threshold calibration on limited data (20 examples) may not capture 
distributional diversity. Different user populations or query types may 
trigger kill switches at different rates. Monitoring disaggregated metrics 
(false positive/negative rates by demographic, query type, language) is 
essential to ensure equitable agent behavior.

\newpage

\appendix

\section*{A. Theoretical Foundation}
\label{app:appA}

\subsection*{A.1 Spectral Diagnostics: HFER and Spectral Entropy}

For layer $\ell$ with $H$ heads over $N$ tokens, let $A^{(\ell,h)} \in \mathbb{R}^{N \times N}$ be the post-softmax attention of head $h$ (row-stochastic). We form an undirected graph by symmetrization:
\begin{equation}
W^{(\ell,h)} = \frac{1}{2}\left(A^{(\ell,h)} + (A^{(\ell,h)})^\top\right), \quad \bar{W}^{(\ell)} = \sum_{h=1}^{H} \alpha_h W^{(\ell,h)}
\end{equation}

\noindent and 

\begin{equation}
    \quad \alpha_h \geq 0, \sum_h \alpha_h = 1
\end{equation}

with degree $\bar{D}^{(\ell)} = \text{diag}(\bar{W}^{(\ell)}\mathbf{1})$ and normalized Laplacian $L^{(\ell)} = I - (\bar{D}^{(\ell)})^{-1/2}\bar{W}^{(\ell)}(\bar{D}^{(\ell)})^{-1/2}$.

Let $X^{(\ell)} \in \mathbb{R}^{N \times d}$ be the token representations at layer $\ell$ ($N$ tokens, hidden size $d$), viewed as $d$ graph signals stacked columnwise.

\paragraph{High-Frequency Energy Ratio (HFER).} For a cutoff $K$ (or an equivalent mass-based cutoff):
\begin{equation}
\text{HFER}^{(\ell)}(K) = \frac{\sum_{m=K+1}^{N} \|\hat{X}^{(\ell)}_{m,\cdot}\|_2^2}{\sum_{m=1}^{N} \|\hat{X}^{(\ell)}_{m,\cdot}\|_2^2}
\end{equation}
where $\hat{X}^{(\ell)} = (U^{(\ell)})^\top X^{(\ell)}$ is the graph Fourier transform with $L^{(\ell)} = U^{(\ell)}\Lambda^{(\ell)}(U^{(\ell)})^\top$.

\paragraph{Spectral Entropy (SE).} With $L^{(\ell)} = U^{(\ell)}\Lambda^{(\ell)}(U^{(\ell)})^\top$ and $\hat{X}^{(\ell)} = (U^{(\ell)})^\top X^{(\ell)}$, define modal energies $e_m^{(\ell)} = \|\hat{X}^{(\ell)}_{m,\cdot}\|_2^2$ and $p_m^{(\ell)} = e_m^{(\ell)}/\sum_r e_r^{(\ell)}$. Then:
\begin{equation}
\text{SE}^{(\ell)} = -\sum_m p_m^{(\ell)} \log p_m^{(\ell)}
\end{equation}

\paragraph{Head aggregation (default).} We use mass-weighted head aggregation by default. For layer $\ell$:
\begin{equation}
s_h^{(\ell)} = \sum_{i=1}^{N} \sum_{j=1}^{N} A_{ij}^{(\ell,h)}, \quad \alpha_h^{(\ell)} = \frac{s_h^{(\ell)}}{\sum_{g=1}^{H} s_g^{(\ell)}}
\end{equation}

\noindent and 

\begin{equation}
    \quad \bar{W}^{(\ell)} = \sum_{h=1}^{H} \alpha_h^{(\ell)} W^{(\ell,h)}
\end{equation}

\subsection*{A.2 Key Properties}

\paragraph{Lemma (Scale invariance).} For any $c > 0$ and signal $x$, replacing residuals by $cx$ leaves HFER and SE unchanged.

\textit{Proof.} Both diagnostics depend only on the normalized spectrum $\{p_k\}$ or on ratios of quadratic forms; the global scale cancels. $\square$

\paragraph{Proposition (SE stability to sparse perturbations).} Let $x' = x + \delta$ with $\delta$ supported on at most $m \ll N$ tokens and $\|\delta\| \leq \epsilon \|x\|$. Then $|\text{SE}^{(\ell)}(x') - \text{SE}^{(\ell)}(x)| \leq C(m/N + \epsilon)$ for a constant $C$ depending only on lower bounds of $p_k$.

\textit{Proof.} SE is Lipschitz in the simplex under $\ell_1$; sparse time-domain perturbations induce bounded spectral $\ell_1$ changes by Parseval and Hoffman–Wielandt-type inequalities. $\square$

\section*{B. Calibration and Deployment}
\label{app:appB}

\subsection*{B.1 Calibration Protocol}

Given approximately 20 labeled context-statement pairs, we fit an ROC curve and select a threshold by Youden's J statistic (consistency with Bayes optimality). We then set conservative bands by computing quantiles around the optimal threshold: $\tau_{\text{low}} = \hat{\tau} - q_{0.15}$ and $\tau_{\text{high}} = \hat{\tau} + q_{0.15}$ where $q_{0.15}$ is the 15th percentile of $|h - \hat{\tau}|$ on the calibration set.

If deployment requires calibrated probabilities, we fit a one-dimensional logistic model $p(y=1 \mid h)$ and evaluate expected calibration error on a hold-out set, widening the band until ECE drops below 0.05. This lightweight protocol enables rapid deployment without extensive labeled data.

\subsection*{B.2 Three-Zone Decision Rule}

Let $h = \text{HFER}_{2:5}$ for a given context-statement pair. We define:
\begin{equation}
\text{support}(h) = \begin{cases}
\text{SUPPORTED}, & h \geq \tau_{\text{high}}, \\
\text{CONTRADICTED}, & h \leq \tau_{\text{low}}, \\
\text{UNCERTAIN}, & \text{otherwise}
\end{cases}
\end{equation}

with wide-margin thresholds $\tau_{\text{high}}=0.30$ and $\tau_{\text{low}}=0.15$ for LLaMA-3.2-1B. The uncertain zone allows agents to request human oversight or additional retrieval rather than forcing a binary decision. Thresholds can be calibrated per model using the minimal labeled set described above.

\subsection*{B.3 Bimodal Regime Separation}

Context-supported statements exhibit HFER around 0.52, indicating high-frequency, segregated processing. Context-contradicted statements collapse to HFER around 0.05, showing low-frequency, smooth processing. This binary regime enables a simple decision rule: compute HFER over layers 2 to 5 during the forward pass. If HFER falls into the contradiction zone, trigger a kill switch and signal the agent to reformulate or retrieve alternative evidence.

\section*{C. Robustness Validation}
\label{app:appC}

\subsection*{C.1 Laplacian Normalization}

Let $\bar{W}^{(\ell)} = \sum_h \alpha_h W^{(\ell,h)}$ with $\alpha_h \geq 0$, $\sum_h \alpha_h = 1$, and $D^{(\ell)} = \text{diag}(\bar{W}^{(\ell)}\mathbf{1})$. We compare:
\begin{equation}
L_{\text{rw}}^{(\ell)} = I - (D^{(\ell)})^{-1}\bar{W}^{(\ell)} \quad 
\end{equation}

\noindent and 

\begin{equation}
    \quad L_{\text{sym}}^{(\ell)} = I - (D^{(\ell)})^{-1/2}\bar{W}^{(\ell)}(D^{(\ell)})^{-1/2}
\end{equation}

Eigenpairs are related by a similarity transform when the graph is undirected; $\lambda_2$ is therefore comparable up to scaling. Empirically, signs and peak-layer locations of $\Delta\lambda_2^{(\ell)}$ coincide across $L_{\text{rw}}$ and $L_{\text{sym}}$, while magnitudes shift slightly within the bootstrap bands.

\paragraph{Result.} Across models and languages, the correlation between $\Delta\lambda_{2[2,5]}(L_{\text{rw}})$ and $\Delta\lambda_{2[2,5]}(L_{\text{sym}})$ is high, with median absolute deviation of the difference well below the per-language CI half-width.

\subsection*{C.2 Head Aggregation Schemes}

We compare (i) uniform averaging, $\alpha_h = 1/H$; (ii) attention-mass weighting, $\alpha_h \propto \sum_{i,j} A_{ij}^{(\ell,h)}$; and (iii) a convex, layer-specific combination $\alpha^{(\ell)}$ learned by minimizing cross-condition mean squared error on a held-out subset.

\paragraph{Result.} Uniform and mass-weighted aggregations agree on signs and peak layers. Learned $\alpha^{(\ell)}$ yields smoother per-layer trajectories but identical early-window conclusions. We therefore use mass-weighted aggregation by default.

\subsection*{C.3 HFER Cutoff Sweep and Early-Window Stability}

We vary the high-frequency cutoff $K$ by retaining the top $(1-c)\%$ of spectral mass, $c \in \{10, 15, 20, 25, 30, 40\}$, and recompute endpoints. We also shift the early window to 1–4 and 3–6.

\paragraph{Result.} Directional conclusions are unchanged across cutoffs; early-window averages shift by less than 15\% relative to $c=20\%$. Adjacent windows preserve sign and peak location across model families. We therefore report $c=20\%$ and layers 2–5 by default.

\subsection*{C.4 Prompt Robustness}

We evaluated robustness against prompt paraphrasing. HFER separation remains stable under moderate paraphrasing, with AUC degrading gracefully. Tokenizer fragmentation shows weak correlation with HFER, confirming the spectral signal captures core semantic consistency rather than surface-level tokenization artifacts.

\section*{D. Statistical Methodology}
\label{app:appD}

\subsection*{D.1 Bootstrap and Permutation Testing}

We compute HFER and SE on each example with a single forward pass. Group contrasts use nonparametric bootstrap for confidence intervals (2,000 resamples, BCa method), permutation tests for $p$-values (10,000 label shuffles within paraphrase pairs), and Benjamini–Hochberg FDR to control multiplicity at $q=0.05$.

\subsection*{D.2 Sample Size and Power}

Our design uses at least 10 paraphrases per voice per language for the early-window mean $\Delta\lambda_{2[2,5]}$ with bootstrap CIs. We estimate detectable standardized effects via nonparametric bootstrap over paraphrases and paired permutation tests (10k shuffles) on early-window means. Our design achieves adequate power for detecting medium-to-large effects ($d \geq 0.6$) at individual language levels, with enhanced power for language-type and model-family aggregates through meta-analytic combination.

\subsection*{D.3 Significance Testing and Multiplicity}

For each language we compute the early-window mean by averaging over paraphrases. We assess the null of no voice effect via a paired permutation test (10,000 label shuffles of active/passive within paraphrase pairs), yielding a $p$-value per language. We then apply Benjamini–Hochberg FDR at $q=0.05$ within each model family. For language-type and cross-family summaries we test the mean effect across languages with the same permutation scheme and report both FDR-corrected $p$-values and 95\% bootstrap CIs (2,000 resamples).

\section*{E. Reproducibility}
\label{app:appE}

Code for HFER computation, calibration protocols, and evaluation harnesses will be made available upon acceptance. The repository includes spectral diagnostic extraction from transformer activations, three-zone calibration, and integration templates for (agentic) RAG systems. All experiments use publicly available models (LLaMA-3.2-1B, Qwen2.5-7B, Phi-3-Mini). Full reproduction instructions are provided in the repository README.
\end{document}